\documentclass{article}
\usepackage{ijcai15}

\usepackage{times}

\newcommand{\etal}{\emph{et~al.}}

\usepackage{amsmath} \usepackage{amssymb} \usepackage{amsthm}

\DeclareMathOperator*{\argmin}{arg\,min}

\newcommand{\E}{\mathbb{E}}
\newcommand{\EE}[1]{\E\!\left[ #1 \right]}

\newcommand{\Event}[1]{\mathbb{I}\left\lbrace #1 \right\rbrace}

\newcommand{\ceiling}[1]{\left\lceil #1 \right\rceil }
\newcommand{\floor}[1]{\left\lfloor #1 \right\rfloor }

\newcommand{\Naturals}{\mathbb{N}}

\newcommand{\Reals}{\mathbb{R}}

\newtheorem{theorem}{Theorem}
\newtheorem{definition}{Definition}

\newtheorem{corollary}[theorem]{Corollary}

\theoremstyle{definition}

\usepackage{enumerate}

\usepackage{tikz} \usetikzlibrary{arrows}
\usepackage{varwidth}
\usepackage{multirow}
\usepackage{booktabs}

\usepackage{algorithmic}
\usepackage{algorithm}
\usepackage{url}
\usepackage{adjustbox}
\usepackage{graphicx}
\usepackage{xspace}

\title{Fast Cross-Validation for Incremental Learning}
\author{Pooria Joulani \qquad Andr\'{a}s Gy\"{o}rgy \qquad Csaba Szepesv\'{a}ri\\
Department of Computing Science, University of Alberta\\
Edmonton, AB, Canada \\
\{pooria,gyorgy,szepesva\}@ualberta.ca}

\begin{document}

\maketitle

\allowdisplaybreaks 
\newcommand{\cvAlg}{\textsc{TreeCV}\xspace}
\newcommand{\learningAlg}{\mathcal{L}}
\newcommand{\Call}[2]{\textsc{#1}\left( #2 \right)}
\newcommand{\tCall}[2]{\textsc{#1}( #2 )}
\newcommand{\strt}{s} \newcommand{\nd}{e}
\newcommand{\mdl}{\hat{f}_{\strt..\nd}}
\newcommand{\nomodel}{\emptyset}
\newtheorem{assumption}{Assumption}

\begin{abstract}
  Cross-validation (CV) is one of the main tools for performance
  estimation and parameter tuning in machine learning. 
  The general recipe for computing CV estimate
  is to run a learning algorithm separately 
  for each CV fold, a computationally expensive process. 
  In this paper, we propose a new approach to reduce 
  the computational burden of CV-based performance estimation. 
  As opposed to all previous attempts, which are specific to a particular
  learning model or problem domain, we propose a general method applicable 
  to a large class of incremental learning algorithms, 
  which are uniquely fitted to big data problems.
  In particular, our method applies to a wide
  range of supervised and unsupervised learning tasks with different
  performance criteria, as long as the base learning algorithm is
  incremental. We show that the running time of the algorithm scales
  logarithmically, rather than linearly, in the number of CV
  folds. Furthermore, the algorithm has favorable properties for
  parallel and distributed implementation. Experiments with
  state-of-the-art incremental learning algorithms confirm the
  practicality of the proposed method.
\end{abstract}

\section{Introduction}
\label{sec:intro}
Estimating generalization performance is a core task in machine learning.
Often, such an estimate is computed using 
$k$-fold cross-validation ($k$-CV): the dataset is
partitioned into $k$ subsets of approximately equal size, 
and each subset is used to evaluate a model trained on the $k-1$ other subsets
to produce a numerical score;
the $k$-CV performance estimate is then obtained as the average of the obtained
scores. 

A significant drawback of $k$-CV is its heavy computational cost.
The standard method for computing a $k$-CV estimate is to train $k$
separate models independently, one for each fold, requiring (roughly) 
$k$-times the work of training a single model. 
The extra computational cost imposed by $k$-CV is especially high for 
leave-one-out CV (LOOCV), a popular variant, 
where the number of folds equals the number of samples in the dataset.
The increased computational requirements may become a major problem, especially when CV is used for tuning
hyper-parameters of learning algorithms in a grid search, in which
case one $k$-CV session needs to be run for every combination of
hyper-parameters, dramatically increasing the computational cost even when the number of
hyper parameters is small.\footnote{ For example, the semi-supervised
  anomaly detection method of
  G\"{o}rnitz \etal\ \shortcite{gornitz2013supervisedanomalydetectionb} has four
  hyper-parameters to tune. Thus, testing all possible combinations for, e.g., 10 possible values of each hyper-parameter requires running CV 10000 times.}

To avoid the added cost, much previous research went into studying the
efficient calculation of the CV estimate  (exact or approximate).
However, previous work has been concerned with special models and
problems: With the exception of Izbicki \shortcite{Izb13}, these methods are typically limited
to linear prediction with the squared loss
and to kernel methods with various loss functions, including twice-differentiable
losses and the hinge loss (see Section~\ref{sec:related-work} for details). %
In these works, the training time of the underlying learning algorithm
is $\Theta(n^3)$, where $n$ is the size of the dataset, and the
main result states that the CV-estimate (including LOOCV estimates) is yet computable 
in $O(n^3)$ time. 
Finally, Izbicki \shortcite{Izb13} gives a very efficient
solution (with $O(n+k)$ computational complexity) for the restrictive case when
two models trained on any two datasets can be combined, in constant
time, into a single model that is trained on the union of the datasets.

Although these results are appealing,
they are limited to methods and problems with specific features.
In particular, they are unsuitable for big data problems where 
the only practical methods are incremental and run in  
linear, or even sub-linear time
\cite{shalev-shwartz2011pegasosprimalestimated,ClHaWo12}.
In this paper, we show that CV calculation can be done efficiently for incremental learning algorithms.
In Section~\ref{sec:tree-cv}, we present a method that, under mild,
natural conditions, speeds up the calculation of the $k$-CV estimate
for incremental learning algorithms, in the general learning setting
explained in Section~\ref{sec:problem} (covering a wide range of supervised and unsupervised learning problems), and for arbitrary performance
measures. The proposed method, \cvAlg, exploits
the fact that incremental learning algorithms do not need to be fed with the whole dataset at once,
but instead learn from whatever data they are provided with and later
update their models when more data arrives, without the need to be
trained on the whole dataset from scratch. As we will show in Section~\ref{sec:analysis}, \cvAlg computes a guaranteed-precision approximation of the CV estimate when the algorithms produce \emph{stable} models.
We present several implementation details and analyze the time and space complexity of \cvAlg in
Section~\ref{sec:implementation}. In particular, we show that its computation time is only $O(\log k)$-times bigger than the time required to train a single model, which is a major
improvement compared to the $k$-times increase required for a naive computation of the CV estimate.
Finally, Section~\ref{sec:experiments}  presents experimental results, which confirm the efficiency of the proposed algorithm.

\subsection{Related Work}
\label{sec:related-work}
Various methods, often specialized to specific learning settings, have
been proposed to speed up the computation of the $k$-CV estimate. Most
frequently, efficient $k$-CV computation methods are specialized to
the regularized least-squares (RLS) learning settings (with squared-RKHS-norm regularization).
In particular, the generalized cross-validation method
\cite{golub1979generalizedcrossvalidation,wahba1990splinemodelsobservational}
computes the LOOCV estimate in $O(n^2)$ time for a dataset of size
$n$ from the solution of the RLS problem over the whole dataset;
this is generalized to $k$-CV calculation in $O(n^3/k)$ time by
Pahikkala \etal\ \shortcite{pahikkala2006fastnfolda}.  In the special case of least-squares
support vector machines (LSSVMs), Cawley~\shortcite{cawley2006leaveoneout} shows
that LOOCV can be computed in $O(n)$ time using a Cholesky
factorization (again, after obtaining the solution of the RLS
problem).
It should be noted that all of the aforementioned methods
use the inverse (or some factorization) of a special matrix (called
the \emph{influence matrix}) in their calculation; the aforementioned
running times are therefore based on the assumption that this inverse
is available (usually as a by-product of solving the RLS problem, computed in
$\Omega(n^3)$ time).\footnote{In the absence of this assumption,
  stochastic trace estimators \cite{girard1989fastmontecarlo} or
  numerical approximation techniques
  \cite{golub1997generalizedcrossvalidation,nguyen2001efficientgeneralizedcross} are used to avoid the costly
  inversion of the matrix.}  %

A related idea for approximating the LOOCV estimate is using the
notion of \emph{influence
  functions}, %
which measure the effect of adding an infinitesimal single point of
probability mass to a distribution. Using this notion,
Debruyne \etal\ \shortcite{debruyne2008modelselectionkernel} propose to approximate the
LOOCV estimate for kernel-based regression algorithms that use any
twice-differentiable loss function.
Liu \etal\ \shortcite{liu2014efficientapproximationcross}
use \emph{Bouligand influence functions}
\cite{christmann2008bouligandderivativesrobustness}, a generalized
notion of influence functions for arbitrary distributions%
, in order to calculate the $k$-CV estimate for kernel methods and 
twice-differentiable loss functions. %
Again, these methods need an existing model trained on the whole dataset, and require $\Omega(n^3)$ running time.

A notable exception to the square-loss/differentiable loss requirement is the work of
Cauwenberghs and Poggio \shortcite{cauwenberghs2001incrementaldecrementalsupport}. They propose an
incremental training method for support-vector classification (with
the hinge loss), and show how to revert the incremental algorithm to
``unlearn'' data points and obtain the LOOCV estimate. The LOOCV
estimate is obtained in time similar to that of a single training by
the same incremental algorithm, which is $\Omega(n^3)$ in the worst
case.

Closest to our approach is the recent work of Izbicki \shortcite{Izb13}: assuming
that two models trained on any two separate datasets can be combined, in
constant time, to a single model that is exactly the same as if the
model was trained on the union of the
datasets, Izbicki \shortcite{Izb13} can compute the $k$-CV estimate in $O(n+k)$
time. However his assumption is very restrictive and applies only to
simple methods, such as Bayesian classification.\footnote{
The other methods considered by Izbicki \shortcite{Izb13} do not satisfy the
theoretical assumptions of that paper.}
In contrast, roughly, we only assume that
a model can be updated efficiently with new data (as opposed to
combining the existing model and a model trained on the new data in constant time), and
we only require that models trained with permutations of the data be
sufficiently similar, not exactly the same.

Note that the CV estimate depends on the specific partitioning of the
data on which it is calculated. To reduce the variance due to
different partitionings, the $k$-CV score can be averaged over
multiple random partitionings. For LSSVMs,
An \etal\ \shortcite{an2007fastcrossvalidation} propose a method to efficiently
compute the CV score for multiple partitionings, resulting in a total
running time of $O(L(n-b)^3)$, where $L$ is the number of different
partitionings and $b$ is the number of data points in each test set.
In the case when all possible partitionings of the dataset are used,
the complete CV (CCV) score is obtained.
Mullin and Sukthankar \shortcite{mullin2000completecrossvalidation} study efficient computation
of CCV for nearest-neighbor-based methods; their method runs in time
$O(n^2k+n^2\log(n))$.

\section{Problem Definition}
\label{sec:problem}
We consider a general setting that encompasses a wide
range of supervised and unsupervised learning scenarios (see
Table~\ref{tab:special-settings} for a few examples). In this
setting, we are given a dataset $\{z_1, z_2, \ldots, z_n\}$,\footnote{
  Formally, we assume that this is a multi-set, so there might be
  multiple copies of the same data point.  } where each \emph{data
  point} $z_i = (x_i, y_i)$ consists of an \emph{input}
$x_i \in \mathcal{X}$ and an \emph{outcome} $y_i \in \mathcal{Y}$, for
some given sets $\mathcal{X}$ and $\mathcal{Y}$. For example, we might
have $\mathcal{X} \subset \Reals^d, d \ge 1$, with
$\mathcal{Y}=\{+1,-1\}$ in binary classification and
$\mathcal{Y} \subset \Reals$ in regression; for unsupervised learning,
$\mathcal{Y}$ is a singleton:
$\mathcal{Y}=\{{\textsc{NoLabel}}\}$.  We define a \emph{model} as
a function\footnote{Without loss of generality, we only consider
  deterministic models: we may embed any randomness
  required to make a prediction into the value of $x$, so that $f(x)$ is a deterministic
  mapping from $\mathcal{X}$ to $\mathcal{P}$.}
$f:\mathcal{X} \to \mathcal{P}$ that, given an input
$x \in \mathcal{X}$, makes a \emph{prediction}, $f(x) \in \mathcal{P}$,
where $\mathcal{P}$ is a given set (for example,
$\mathcal{P}=\{+1, -1\}$ in binary classification: the model predicts
which class the given input belongs to). Note that the prediction set
need not be the same as the outcome set, particularly for unsupervised
learning tasks. The quality of
a prediction is assessed by a \emph{performance measure} (or \emph{loss function})
$\ell: \mathcal{P} \times \mathcal{X} \times \mathcal{Y} \to
\Reals$
that assigns a scalar value $\ell(p, x, y)$ to the prediction $p$ for
the pair $(x, y)$; for example, $\ell(p, x, y) = \Event{p \neq y}$ for
the prediction error (misclassification rate) in binary classification
(where $\Event{E}$ denotes the indicator function of an event $E$).
\begin{table}[t]
  \centering
  \begin{adjustbox}{width=0.48\textwidth}
  \begin{tabular}[t]{cccccc}
    \hline
    Setting & $\mathcal{X}$ & $\mathcal{Y}$ & $\mathcal{P}$ & $\ell(f(x), x, y)$
    \\ \hline
    \multirow{2}{*}{Classification} & \multirow{2}{*}{$\Reals^d$} & \multirow{2}{*}{$\{+1, -1\}$} & \multirow{2}{*}{$\{+1, -1\}$} & \multirow{2}{*}{$\Event{f(x) \neq y}$} \\
    &&&& \\
    Regression & $\Reals^d$ & $\Reals$ & $\Reals$ & $(f(x)-y)^2$ \\
    \multirow{2}{*}{$K$-means clustering} & \multirow{2}{*}{$\Reals^d$} & \multirow{2}{*}{$\{{\textsc{NoLabel}}\}$} & \multirow{2}{*}{$\{c_1,c_2,\ldots,c_K\}\subset \Reals^d$} & \multirow{2}{*}{$\| x - f(x)\|^2$ } \\
&&&&
    \\ 
    \multirow{2}{*}{\shortstack{Density estimation}} & \multirow{2}{*}{$\Reals^d$} & \multirow{2}{*}{$\{\textsc{NoLabel}\}$} & \multirow{2}{*}{$\{f: f \textnormal{ is a density}\}$} & \multirow{2}{*}{$- \log(f(x))$} \\
            &&&&
    \\ 
\hline
  \end{tabular}
  \end{adjustbox}
  \caption{Instances of the general learning problem considered in the paper. In $K$-means clustering, $c_j$ denotes the center of the $j$th cluster.}
  \label{tab:special-settings}
\end{table}

Next, we define the notion of an \emph{incremental learning
  algorithm}. Informally, an incremental learning algorithm is a
procedure that, given a model learned from previous data points and a
new dataset, updates the model to accommodate the new dataset
at the fraction of the cost of training the model on the whole data
from scratch.
Formally, let
$\mathcal{M} \subseteq \left\lbrace f:\mathcal{X} \to \mathcal{P}
\right\rbrace$
be a set of models, %
and define $\mathcal{Z}^{*}$
to be the set of all possible datasets of all possible sizes. 
Disregarding computation for now, an incremental learning algorithm is a mapping
$\learningAlg: \left( \mathcal{M} \cup \{\nomodel\} \right) \times
\mathcal{Z}^{*} \to \mathcal{M}$
that, given a model $f$ from $\mathcal{M}$ (or $\nomodel$ when a model
does not exist yet) and a dataset $Z' = (z'_1, z'_2, \ldots, z'_m)$,
returns an ``updated'' model $f' = \learningAlg(f,Z')$.
To capture often needed  internal states 
(e.g., to store learning rates), we allow the ``padding'' of
the models in $\mathcal{M}$ with extra information as necessary,
while still viewing the models as $\mathcal{X} \to \mathcal{P}$ maps when convenient.
Above, $f$
is usually the result of a previous invocation of $\learningAlg$ on another
dataset $Z \in \mathcal{Z}^{*}$. In particular,
$\learningAlg(\nomodel, Z)$ learns a model from scratch using the dataset $Z$. An important class of incremental algorithms are \emph{online} algorithms, which update the model one data point at a time: to update $f$ with $Z'$, these algorithms make $m$ consecutive calls to $\learningAlg$, where each call updates the latest model with the next remaining data point according to a random ordering of the points in $Z'$.

In the rest of this paper, we consider an incremental learning
algorithm $\learningAlg$, and a fixed, given partitioning of the dataset $\{z_1, z_2, \ldots, z_n\}$ into $k$ subsets (``chunks'')
$Z_1, Z_2, \ldots, Z_k$. We use $f_i = \learningAlg(\nomodel, Z \setminus Z_i)$ to denote the model learned from
all the chunks except $Z_i$.  Thus, the $k$-CV estimate of the
generalization performance of $\learningAlg$, denoted
$R_{k\textnormal{-CV}}$, is given by
\begin{align*}
  R_{k\textnormal{-CV}} = \dfrac{1}{k} \sum_{i=1}^{k} R_i,
\end{align*}
where
$R_i = \frac{1}{|Z_i|} \sum_{(x,y) \in Z_i} \ell ( f_i(x), x, y),
i=1,2,\ldots, k,$
is the performance of the model
$f_i$ evaluated on $Z_i$.
The LOOCV estimate $R_{n\textnormal{-CV}}$ is obtained when $k=n$.

\section{Recursive Cross-Validation}
\label{sec:tree-cv}
Our algorithm builds on the observation that for every $i$ and $j$,
$1 \le i < j \le k$, the training sets $Z\setminus Z_i$ and
$Z\setminus Z_j$ are almost identical, except for the two chunks $Z_i$
and $Z_j$ that are held out for testing from one set but not the
other. The naive $k$-CV calculation method ignores this fact,
potentially wasting computational resources.  
When using an incremental learning algorithm, we may be able to exploit this
redundancy: we can first learn a model only from the examples shared
between the two training sets, and then ``increment'' the differences
into two different copies of the model learned. 
When the extra cost of saving and restoring a model required by this approach
is comparable to learning a model from scratch, then this approach may 
result in a considerable speedup.

\begin{algorithm}[t!]
  \caption{$\cvAlg \left( \strt, \nd, \mdl \right)$}
  \label{alg:tree-cv}
  \begin{algorithmic}
    \STATE \textbf{input:} indices $\strt$ and $\nd$, and the model
    $\mdl$ trained so far. %
    \IF{$\nd = \strt$} %
    \STATE
    $\hat{R}_{\strt} \gets \frac{1}{|Z_{\strt}|} \sum_{(x,y) \in
      Z_{\strt}} \ell \left( \mdl(x), x, y \right).$ %
    \STATE \textbf{return} $\frac{1}{k}\hat{R}_{\strt}$. %
    \ELSE %
    \STATE Let $m \gets \floor{\frac{\strt+\nd}{2}}$. %
    \STATE Update the model with the chunks $Z_{m+1}, \ldots, Z_{\nd}$
    to get
    $\hat{f}_{\strt..m} = \learningAlg(\mdl, Z_{m+1}, \ldots,
    Z_{\nd})$.%
    \STATE Let $r \gets \Call{\cvAlg}{\strt, m,
      \hat{f}_{\strt..m}}$. %
    \STATE Update the model with the chunks $Z_{\strt}, \ldots, Z_m$ to
    get
    $\hat{f}_{m+1..\nd} = \learningAlg(\mdl,Z_{\strt}, \ldots, Z_m )$.%
    \STATE Let
    $r \gets r + \Call{\cvAlg}{m+1, \nd, \hat{f}_{m+1..\nd}}$. %
    \STATE \textbf{return} $r$. %
    \ENDIF %
  \end{algorithmic}
\end{algorithm}

To exploit the aforementioned redundancy in training all $k$ models at
the same time, we organize the $k$-CV computation process in a tree
structure. The resulting recursive procedure,
$\tCall{\cvAlg}{\strt, \nd, \mdl}$, shown in Algorithm~\ref{alg:tree-cv},
receives two indices $\strt$ and $\nd$, $1 \le \strt \le \nd \le k$, and a
model $\mdl$ that is trained on all chunks except
$Z_{\strt}, Z_{\strt+1}, \ldots, Z_{\nd}$, and returns $(1/k)\sum_{i=\strt}^\nd \hat{R}_i$, the normalized sum of the
performance scores $\hat{R}_i, i = \strt,\ldots,\nd,$ corresponding to
testing $\hat{f}_{i..i}$, the model trained on $Z\setminus Z_i$, on the chunk $Z_i$, for $i=\strt,\ldots,\nd$.
\cvAlg divides the hold-out chunks into two groups
$Z_s, Z_{s+1}, \ldots, Z_m$ and $Z_{m+1}, \ldots Z_\nd$, where
$m=\floor{\frac{\strt+\nd}{2}}$ is the mid-point, and obtains the
test performance scores for the two groups separately by
recursively calling itself.  More precisely, \cvAlg first
updates the model by training it on the second group of chunks,
$Z_{m+1}, \ldots, Z_{\nd}$, resulting in the model
$\hat{f}_{\strt..m}$, and makes a recursive call
$\tCall{\cvAlg}{\strt, m, \hat{f}_{\strt..m}}$ to get $(1/k)\sum_{i=\strt}^m \hat{R}_i$.
Then, it repeats the same procedure for the other group of chunks: starting from the
original model $\hat{f}_{\strt..\nd}$ it had received, it updates
the model, this time using the first group of the remaining chunks,
$Z_{\strt}, \ldots, Z_m$, that were previously held out, and calls
$\tCall{\cvAlg}{ m+1, \nd, \hat{f}_{m+1..\nd}}$ to get $(1/k)\sum_{i=m+1}^\nd \hat{R}_i$
(for the second group of chunks). The recursion stops when there is only one hold-out chunk
($s=e$), in which case the performance score $\hat{R}_s$ of the model
$\hat{f}_{s..s}$ (which is now trained on all the chunks except for $Z_s$)
is directly calculated and returned. Calling $\tCall{\cvAlg}{1,n,\nomodel}$ calculates $\hat{R}_{k\textnormal{-CV}} = \frac{1}{k} \sum_{i=1}^{k} \hat{R}_i$.
Figure~\ref{fig:tree-example}
shows an example of the recursive call tree underlying a 
run of the algorithm calculating the LOOCV estimate on a dataset of four data points. Note that the tree structure imposes a new order of feeding the chunks to the learning algorithm, e.g., $z_3$ and $z_4$ are learned before $z_2$ in the first branch of the tree.

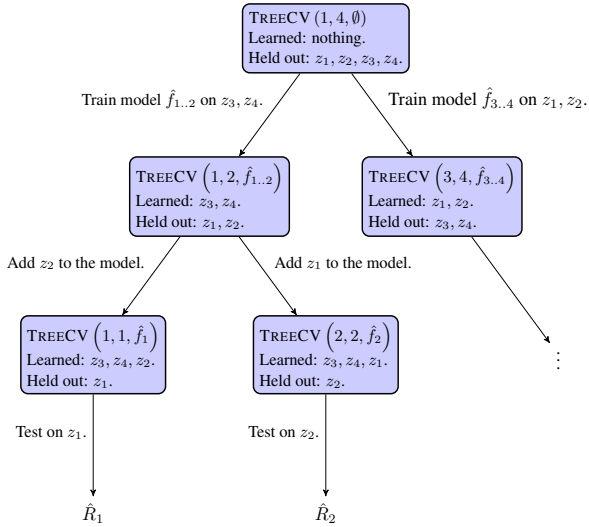
\begin{figure}[t]
  \centering \resizebox{8cm}{7cm}{
    \begin{tikzpicture}[->,>=stealth',level/.style={sibling distance =
        4.5cm, level distance = 3cm, }, auto,]
      \tikzstyle{trnd} = [ rectangle, draw, thick, fill=blue!20, text
      centered, rounded corners, execute at begin
      node={\begin{varwidth}{10.5em}}, execute at end
        node={\end{varwidth}}, ] \node at (0,0) [trnd] { { \small
          $\Call{\cvAlg}{1,4,\nomodel}$\\
          Learned: nothing.\\
          Held out: $z_1, z_2, z_3, z_4$.  } } child{ node [trnd] { {
            \small
            $\Call{\cvAlg}{1,2,\hat{f}_{1..2}}$\\
            Learned: $z_3, z_4$.\\
            Held out: $z_1 ,z_2$.  } } child{ node [trnd] { { \small
              $\Call{\cvAlg}{1,1,\hat{f}_{1}}$\\
              Learned: $z_3, z_4, z_2$.\\
              Held out: $z_1$.  } } child{ node [] {$\hat{R}_1$} edge
            from parent node[above left] { \small Test on $z_1$.  }}
          edge from parent node[above left] { \small Add $z_2$ to the
            model.  } } child{ node [trnd] { { \small
              $\Call{\cvAlg}{2,2,\hat{f}_{2}}$\\
              Learned: $z_3, z_4, z_1$.\\
              Held out: $z_2$.  } } child{ node [] {$\hat{R}_2$} edge
            from parent node[above left] { \small Test on $z_2$.  }}
          edge from parent node[above right] { \small Add $z_1$ to the
            model.  } } edge from parent node[above left] { \small
          Train model $\hat{f}_{1..2}$ on $z_3, z_4$.  } } child{ node
        [trnd] { \small
          $\Call{\cvAlg}{3,4,\hat{f}_{3..4}}$\\
          Learned: $z_1, z_2$.\\
          Held out: $z_3, z_4$.  } child[missing]{} child{ node []
          {$\vdots$} } edge from parent node[above right] { Train
          model $\hat{f}_{3..4}$ on $z_1,z_2$.  } };
    \end{tikzpicture}
  }
  \caption{An example run of \cvAlg on a dataset
    of size four, calculating the LOOCV estimate. 
  }
  \label{fig:tree-example}
\end{figure}

\subsection{Accuracy of \cvAlg}
\label{sec:analysis}
To simplify the analysis, in this section and the next, we assume that each chunk is of the same size, that is $n=k b$ for some integer $b \ge 1$.

Note that the models $\hat{f}_{s..s}$ used in computing $\hat{R}_s$ are
learned incrementally. If the
learning algorithm learns the same model no matter whether it is given
the chunks all at once or gradually, then $\hat{f}_{s..s}$ is the same as the model
$f_s$ used in the definition of $R_{k\textnormal{-CV}}$, and
$R_{k\textnormal{-CV}} = \hat{R}_{k\textnormal{-CV}}$.
If this
assumption does not hold, then $\hat{R}_{k\textnormal{-CV}}$ is still close
to $R_{k\textnormal{-CV}}$ as long as the models $\hat{f}_{s..s}$ are
sufficiently similar to their corresponding models $f_s$. In the rest of this
section, we formalize this assertion.

\newcommand{\fb}{f^{\textnormal{batch}}} 
\newcommand{\finc}{f^{\textnormal{inc}}} 
\newcommand{\Ztr}{Z^{\textnormal{train}}}
\newcommand{\Zte}{Z^{\textnormal{test}}}
\newcommand{\Rte}{R^{\textnormal{test}}}
First, we define the notion of stability for an incremental learning
algorithm. Intuitively, an incremental learning algorithm is stable if
the performance of the models are nearly the same no matter whether
they are learned incrementally or in batch. 
Formally, suppose that a dataset $\{z_1,\ldots,z_n\}$
is partitioned into $l+1$ nonempty chunks $\Zte$ and $\Ztr_1,\ldots, \Ztr_l$, and we are using
$\Zte$ as the test data and the chunks $\Ztr_1, \ldots,\Ztr_l$ as the training data.
Let $\fb=\learningAlg(\nomodel,\Ztr_1 \cup \ldots \cup \Ztr_l)$ denote the model learned
from the training data when provided all at the
same time, and let
\begin{align*}
  \finc=\learningAlg\Bigg(\learningAlg\bigg(\ldots \Big(\learningAlg(\nomodel, \Ztr_1), \Ztr_2
  \Big), \ldots, \Ztr_{l-1} \bigg), \Ztr_l \Bigg)
\end{align*}
denote the model learned from the same chunks when they are provided incrementally
to $\learningAlg$. Let
$\Rte(f) = \frac{1}{|\Zte|} \sum_{(x,y)\in \Zte}
\ell\left(f(x),x,y\right)$
denote the performance of a model $f$ on the test data
$\Zte$. %
\begin{definition}[Incremental stability]
  The algorithm $\learningAlg$ is $g$-incrementally stable for a function
  $g:\Naturals \times \Naturals \to \Reals$
  if, for any dataset $\{z_1, z_2, \ldots, z_n\}$, $b<n$, and partition $\Zte,\Ztr_1,\ldots, \Ztr_l$ with nonempty cells $\Ztr_i, 1\le i \le l$ and $|\Zte|=b$, the test performance of the models $\fb$ and $\finc$ defined above satisfy
  \begin{align*}
    \left| \Rte(\finc) - \Rte(\fb) \right| \le g\left(n-b, b\right).
  \end{align*}
If the data $\{z_1,\ldots,z_n\}$ is drawn independently from the same distribution $\mathcal{D}$ over $\mathcal{X}\times\mathcal{Y}$ and/or the learning algoritm $\learningAlg$ is randomized, we say that $\learningAlg$ is $g$-incrementally stable in expectation if
  \begin{align*}
    \left| \EE{\Rte(\finc)} - \EE{\Rte(\fb)} \right| \le g\left(n-b, b\right)
  \end{align*}
for all partitions selected independently of the data and the randomization of $\learningAlg$.
\end{definition}

The following statement is an immediate consequence of the above definition:
\begin{theorem}
  \label{thm:computes-cv}
  Suppose $n=bk$ for some integer $b \ge 1$  and that algorithm $\learningAlg$ is $g$-incrementally stable.
  Then,
  \begin{align*}
    \left| \hat{R}_{k\textnormal{-CV}} - R_{k\textnormal{-CV}} \right| \le g\left(n-b, b\right).
  \end{align*}
  If $\learningAlg$ is $g$-incrementally stable in expectation then
  \begin{align*}
    \left| \EE{\hat{R}_{k\textnormal{-CV}}} - \EE{R_{k\textnormal{-CV}}} \right| \le g\left(n', b\right).
  \end{align*}
\end{theorem}
\begin{proof}
We prove the first statement only, the proof of the second part is essentially identical.
  Recall that $Z_j, j=1,2,\dots, k$ denote the chunks used for cross-validation. Fix $i$ and let $l=\ceiling{\log k}$. Let $\Zte=Z_i$ and $\Ztr_j, j=1 \ldots l$,  denote the union of the chunks used for training at depth $j$ of the recursion branch ending with the computation of
  $\hat{R}_i$. Then, by definition, $\hat{R}_i = \Rte(\finc)$ and $R_i = \Rte(\fb)$. Therefore, 
  $| \hat{R}_i - R_i | \le g\left(n-b, b\right)$, and the statement follows since $\hat{R}_{k\textnormal{-CV}}$ and $R_{k\textnormal{-CV}}$ are defined as the averages of
the $\hat{R}_i$ and $R_i$, respectively.
\end{proof}

It is then easy to see that incremental learning methods with a bound on their excess risk are incrementally stable in expectation.

\begin{theorem}
\label{thm:excess-risk-stable}
Suppose the data $\{z_1,\ldots,z_n\}$ is drawn independently from the same distribution $\mathcal{D}$ over $\mathcal{X}\times\mathcal{Y}$. Let $(X,Y) \in \mathcal{X}\times\mathcal{Y}$ be drawn from $\mathcal{D}$ independently of the data and let $f^* \in\argmin_{f \in \mathcal{M}} \EE{\ell(f(X), X, Y)}$ denote a model in $\mathcal{M}$ with minimum expected loss. Assume there exist upper bounds $m^{\textnormal{batch}}(n-b)$ and $m^{\textnormal{inc}}(n-b)$ on the excess risks of $\fb$ and $\finc$, trained on $n'=n-b$ data points, such that
\[
\EE{ \ell(\fb(X), X,Y) - \ell( f^*(X), X, Y) } \le m^{\textnormal{batch}}(n')
\]
and
\[
\EE{\ell( \finc(X), X, Y ) - \ell ( f^*(X), X, Y ) } \le m^{\textnormal{inc}}(n')
\]
for all $n$ and for every partitioning of the dataset that is independent of the data, $(X,Y)$, and the randomization of $\learningAlg$. Then $\learningAlg$ is incrementally stable in expectation w.r.t. the loss function $\ell$, with $g(n',b) = \max\{m^{\textnormal{batch}}(n'),m^{\textnormal{inc}}(n')\}$.
\end{theorem}
\begin{proof}
Since the data points in the sets $\Ztr_1,\ldots,\Ztr_l$ and $\Zte$ are independent, $\fb$ and $\finc$ are both independent of $Z^{\textnormal{test}}$. Hence,
$\EE{\Rte(\fb)}=\EE{ \ell \left( \fb(X), X, Y \right) }$ and $\EE{\Rte(\finc)} = \EE{\ell \left( f^{\textnormal{inc}}_n(X), X, Y \right)}$. Therefore,
\begin{align*}
\lefteqn{\EE{\Rte(\finc)}-\EE{\Rte(\fb)} } \\
&= \EE{\Rte(\finc)}-\EE{\ell(f^*(X), X, Y)} \\
&\quad+\EE{\ell(f^*(X), X, Y)} - \EE{\Rte(\fb)} \\
&\le 
\EE{\Rte(\finc)}-\EE{\ell(f^*(X), X, Y)} 
\le m^{\textnormal{inc}}(n')
\end{align*}
where we used the optimality of $f^*$. Similarly, $\EE{\Rte(\fb)}-\EE{\Rte(\finc)} \le m^{\textnormal{batch}}(n')$, finishing the proof.
\end{proof}

In particular, for online learning algorithms satisfying some regret bound, standard online-to-batch conversion results \cite{CBCoGe04,KaTe09} yield excess-risk bounds for independent and identically distributed data. Similarly, excess-risk bounds are often available for stochastic gradient descent (SGD) algorithms which scan the data once (see, e.g., \cite{nemirovski2009robust}). For online learning algorithms (including single-pass SGD), the batch version is usually defined by running the algorithm using a random ordering of the data points or sampling from the data points with replacement. Typically, this version also satisfies the same excess-risk bounds. Thus, the previous theorem shows that these algorithms are
are incrementally stable with $g(n,b)$ being their excess-risk bound for $n$ samples.

Note that this incremental stability is w.r.t. the loss function whose excess-risk is bounded. For example, after visiting $n$ data points, the regret of PEGASOS~\cite{shalev-shwartz2011pegasosprimalestimated} with bounded features is bounded by $O(\log(n))$. Using the online-to-batch conversion of Kakade and Tewari \shortcite{KaTe09}, this gives an excess risk bound $m(n) = O(\log(n)/n)$, and hence PEGASOS is stable w.r.t. the regularized hinge loss with $g(n,b) = m(n)=O(\log(n)/n)$. Similarly, SGD over a compact set with bounded features and a bounded convex loss is stable w.r.t. that convex loss with $g(n,b) = O(1/\sqrt{n})$ \cite{nemirovski2009robust}.
Experiments with these algorithms are shown in Section~\ref{sec:experiments}. Finally, we note that algorithms like PEGASOS or SGD could also be used to scan the data multiple times. In such cases, these algorithms would not be useful incremental algorithms, as it is not clear how one should add a new data point without a major retraining over the previous points. Currently, our method does not apply to such cases in a straightforward way.

\section{Complexity Analysis}
\label{sec:implementation}
In this section, we analyze the running time and storage requirements of \cvAlg, and discuss some practical issues concerning its implementation, including parallelization.

\subsection{Memory Requirements}

%
%
%
%
%
%
%
%
%
%
%

Efficient storage of and updates to the model are crucial for the
efficiency of Algorithm~\ref{alg:tree-cv}: Indeed, in any call of $\tCall{\cvAlg}{\strt, \nd, \mdl}$ that does not correspond to simply evaluating a model on a chunk of data (i.e., $\strt\neq\nd$), \cvAlg has to update the original model $\mdl$ twice, once with $Z_\strt,\ldots,Z_m$, and once with $Z_{m+1},\ldots,Z_\nd$. 
To do this, $\cvAlg$ can either store $\mdl$, or revert to $\mdl$ from $\hat{f}_{\strt..m}$.  
In general, for any type of
model, if the model for $\mdl$ is modified in-place, then we
need to create a copy of it before it is updated to the model 
for $\hat{f}_{s..m}$, or, alternatively, keep track of the changes
made to the model during the update. Whether to use the copying or the save/revert strategy depends on the
application and the learning algorithm. For example, if the model
state is compact, copying is a useful strategy, whereas when
the model undergoes few changes during an update, save/revert might
be preferred. 

Compared to a single run of the learning algorithm $\learningAlg$,
\cvAlg requires some extra storage for saving and
restoring the models it trains along the way. When no
parallelization is used in implementing \cvAlg,
we are in exactly one branch at every point during the execution
of the algorithm.  Since the largest height of a recursion branch is
of $O(\log k)$, and one model (or the changes made to it) is saved in
each level of the branch, the total storage required by \cvAlg
is $O(\log(k))$-times the storage needed for a single model.

\cvAlg can be easily parallelized by dedicating one thread of computation to each of the data groups used in updating $\mdl$ in one call of $\cvAlg$.
In this case one typically needs to copy the model since the
two threads are needed to be able to run independently of each other; thus, the total number of models \cvAlg
needs to store is $O(k)$, since there are $2k-1$ total
nodes in the recursive call tree, with exactly one model stored per
node. Note that a standard parallelized CV calculation also needs to
store $O(k)$ models.

Finally, note that \cvAlg is potentially useful in distributed environment, where each chunk of the data is stored on a different node in the network. Updating the model on a given chunk can then be relegated to that computing node (the model is sent to the processing node, trained and sent back, i.e., this is not using all the nodes at once), and it is
only the model (or the updates made to the model), not the data, that
needs to be communicated to the other nodes.  Since at every level of the tree, each chunk is added to exactly one model, the total communication cost of doing this is $O(k \log(k) )$.

\subsubsection{Running Time}
\label{sec:runtime}

Next, we analyze the time complexity of $\cvAlg$ when
calculating the $k$-CV score for a dataset of size $n$
under our previous simplifying assumption that $n=bk$ for some integer $b \ge 1$.

The running time of \cvAlg is analyzed in terms of
the running time of the learning algorithm $\learningAlg$ and the time
it takes to copy the models (or to save and then revert the
changes made to it while it is being updated by
$\learningAlg$). Throughout this subsection, we use the following
definitions and notations: for $m= 0,1,\ldots, n$, $l = 1,\ldots,n-m$,
and $j=1,\ldots,k$,
\begin{itemize}
\item $t_u(m,l) \ge 0$ denotes the time required to update a model,
  already trained on $m$ data points, with a set of $l$ additional
  data points;
\item $t_s(m,l) \ge 0$ is the time required to copy the model,
  (or save and revert the changes made to it) when the model is already
  trained on $m$ data points and is being updated with $l$ more data
  points;
\item $t(j)$ is the time spent in saving, restoring, and
  updating models in a call to
  $\Call{\cvAlg}{\strt, \nd, \mdl}$ with $j = \nd - \strt + 1$
  hold-out chunks (and with $\mdl$ trained on $k-j$ chunks);
\item $t_{\ell}$ denotes the time required to test a model on one of
  the $k$ chunks (where the model is trained on the other $k-1$
  chunks);
\item $T(j)$ denotes the \emph{total} running time of $\tCall{\cvAlg}{\strt, \nd, \mdl}$ when the number of  chunks held out is $j = \nd - \strt + 1$, and $\mdl$ is already trained with $n-bj$ data points. Note that $T(k)$ is the total running time of \cvAlg to calculate the $k$-CV score for a dataset of size $n$.
\end{itemize}

By definition, for all $j=2 \ldots k$, we have
\begin{align*}
  t(j) &=  t_u(n-bj, b\floor{j/2}) +  t_s(n-bj, b\floor{j/2})  \\
               &\quad + t_u(n-bj, b\ceiling{j/2}) +  t_s(n-bj, b\ceiling{j/2}) + t_c,
\end{align*}
where $t_c \ge 0$ accounts for the cost of the operations other than the recursive function calls.

We will analyze the running time of \cvAlg under the following natural assumptions: First, we assume that
$\learningAlg$ is not slower if data points are provided in batch rather than one by one. That is, 
\begin{align}
  t_u(m, l) \le \sum_{i=m}^{m+l-1} t_u(i,1), \label{eq:admissible-incremental}
\end{align}
for all $m =0,\ldots,n$ and $l = 1,\ldots,n-m$.%
\footnote{If this is not the case, we would always input the data one by one even
  if there are more data points available.}  
Second, we assume that updating a model requires work comparable to 
saving it or reverting the changes made to it during the update.
This is a natural assumption since the update procedure is also writing those changes. Formally, we assume that
there is a constant $c \ge 0$ (typically $c<1$) such that for all $m=0,\ldots,n$ and
$l=1,\ldots,n-m,$
\begin{align}
  t_s(m,l) \le c \ t_u(m,l). \label{eq:admissible-saving}
\end{align}

To get a quick estimate of the running time, assume for a moment the idealized case that $k=2^d$, $t_u(m,l)=l t_u(0,1)$ for all $m$ and $l$, and $t_c=0$. Since $n2^{-j}$ data points are added to the models of a node at level $j$ in the recursive call tree, the work required in such a node is $(1+c) n 2^{-j} t_u(0,1)$. There are $2^j$ such nodes, hence the cumulative running time at level $j$ nodes is $(1+c) n t_u(0,1)$, hence the total running time of the algorithm is $(1+c) n t_u(0,1) \log_2 k$, where
$\log_2$ denotes base-$2$ logarithm. 

The next theorem establishes a similar logarithmic penalty (compared to the
running time of feeding the algorithm with one data point at a time)
in the general case.
\begin{theorem}
  \label{thm:general-incremental}
  Assume \eqref{eq:admissible-incremental} and \eqref{eq:admissible-saving}  are satisfied. Then the total running time of \cvAlg can be bounded as
  \begin{align*}
    T(k) \le &\ n (1+c) t_u^* \log_2(2k) + (k-1)  t_c + k t_{\ell},
  \end{align*}
  where $t_u^* = \max_{0 \le i \le n-1} t_u(i,1)$.
\end{theorem}
\begin{proof}
By \eqref{eq:admissible-incremental},
$t_u(n-bj, l) \le \sum_{i=0}^{l-1} t_u(n-bj+i, 1) \le l\ t_u^*$
for all $l = 1,\ldots,bj$. Combining with \eqref{eq:admissible-saving}, for any $2 \le j \le k$ we obtain
\begin{align}
    t(j) &\le  \ (1+c) t_u(n-bj, b\floor{j/2})  \nonumber \\
    & \quad + \ (1+c) t_u(n-bj, b\ceiling{j/2}) +t_c\nonumber \\
& \le (1+c)bt_u^*\left(\floor{j/2} + \ceiling{j/2}\right) +t_c \nonumber \\
    &=   \frac{(1+c)n}{k}\ j\ t_u^* + t_c := a j +t_c \label{eq:tjbound}
\end{align}
where $a=(1+c) n t_u^*/k$.
Next we show by induction that for $j\ge 2$ this implies
\begin{align}
\label{eq:ind}
T(j) \le a j (\log_2(j-1)+1) +(j-1) t_c +j t_\ell.
\end{align}
Substituting $j=k$ in \eqref{eq:ind} proves the theorem since $\log_2(j-1)+1 \le \log_2(2j)$.
By the definition of \cvAlg,
  \begin{align*}
    T(j) =
    \begin{cases}
      T\left(\floor{\frac{j}{2}}\right) + T\left(\ceiling{\frac{j}{2}}\right) + t(j), & j \ge 2;\\
      t_{\ell}, & j = 1.
    \end{cases} %
  \end{align*}
This implies that \eqref{eq:ind} holds for $j=2,3$. Assuming \eqref{eq:ind} holds for all $2\le j'<j$, $4 \le j \le k$, 
from~\eqref{eq:tjbound} we get 
 \begin{align*}
    T(j) & =  T(\floor{j/2}) + T(\ceiling{j/2}) + t(j) \\
   & \le a j \left( \log_2(\ceiling{j/2} - 1) + 2 \right) + t_c(j-1) + j t_{\ell} \\
    & \le   a j (\log_2(j-1)+1) + t_c(j-1) +j t_\ell
  \end{align*}
  completing the proof of \eqref{eq:ind}. %
\end{proof}

For fully incremental, linear-time learning algorithms (such as PEGASOS or single-pass SGD), we obtain the following upper bound:
\begin{corollary}
  \label{cor:full-incremental}
  Suppose that the learning algorithm $\learningAlg$ satisfies \eqref{eq:admissible-saving}  and 
  $t_u(0, m) = m t_u^*$ for some $t_u^*>0$ and all $1\le m \le n$.
Then $$T(k) \le (1+c) T_{\learningAlg} \log_2(2k) + t_c (k-1) + k t_\ell,$$ where $T_\learningAlg=t_u(0,n)$ is the running time of a single run of $\learningAlg$.
\end{corollary}

\section{Experiments}
\label{sec:experiments}

\newcommand{\iw}{0.30}
\begin{figure*}
\centering
\begin{tabular}{ccc}
\includegraphics[width=\iw\textwidth]{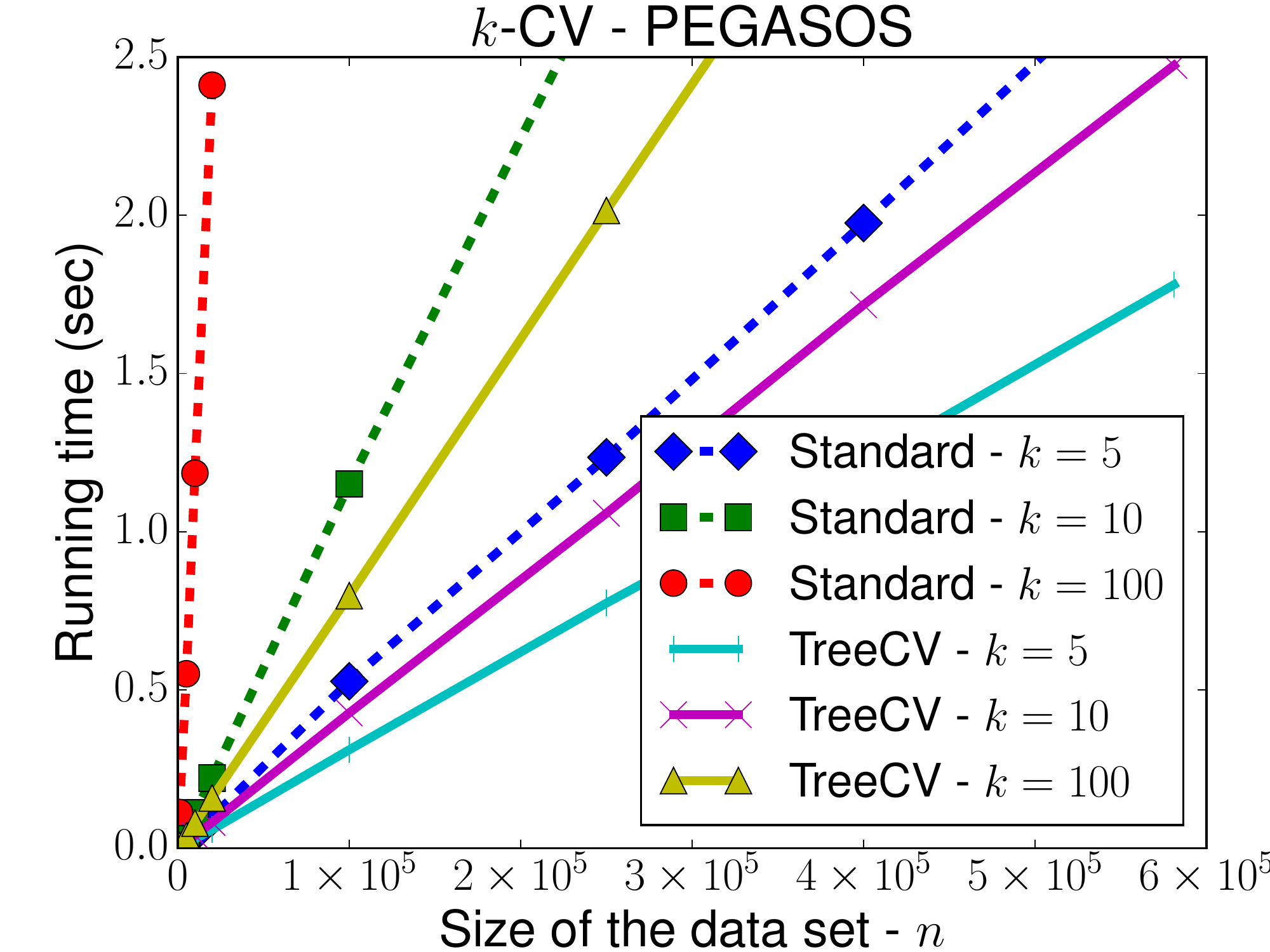}&
\includegraphics[width=\iw\textwidth]{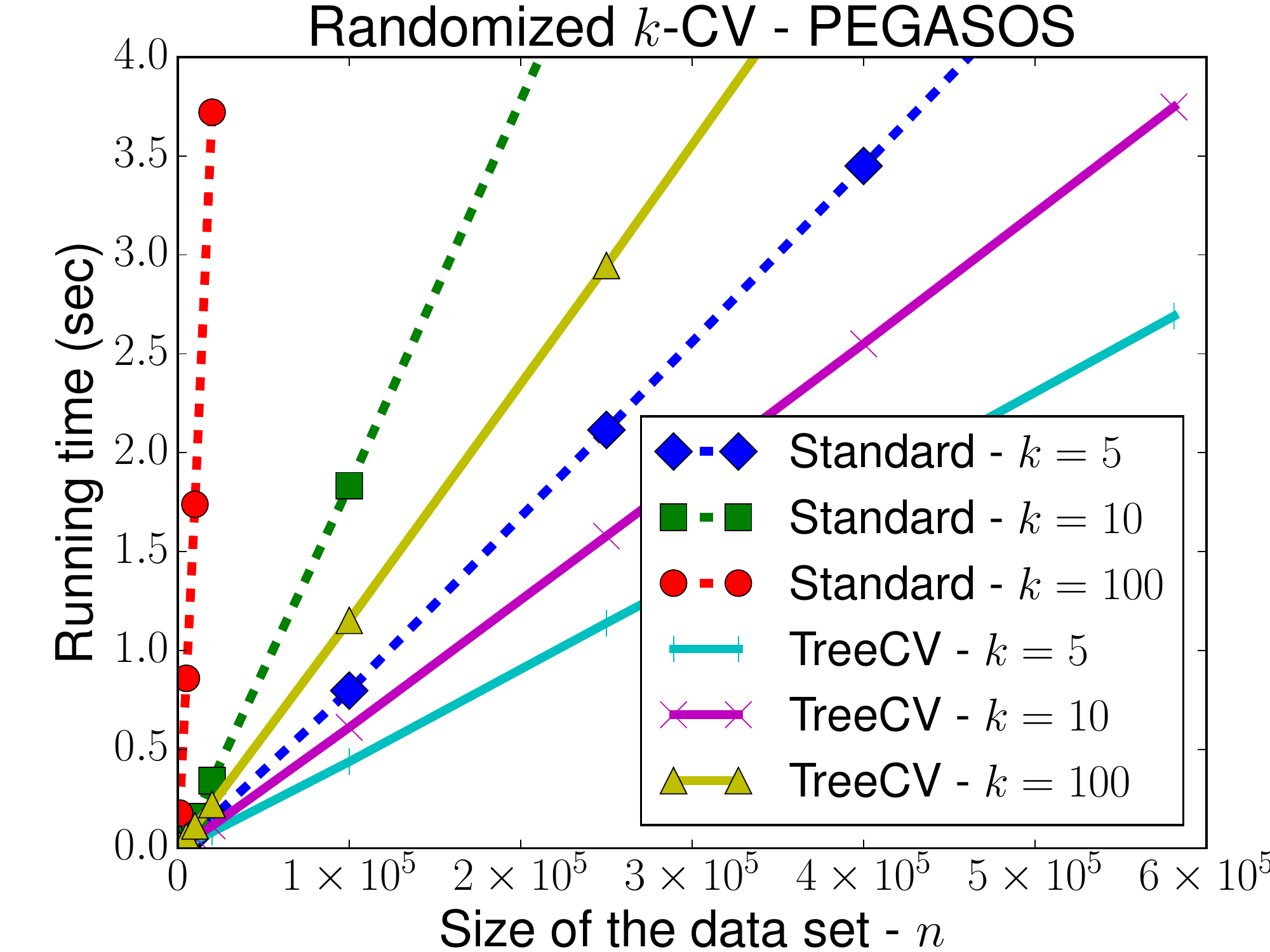}&
\includegraphics[width=\iw\textwidth]{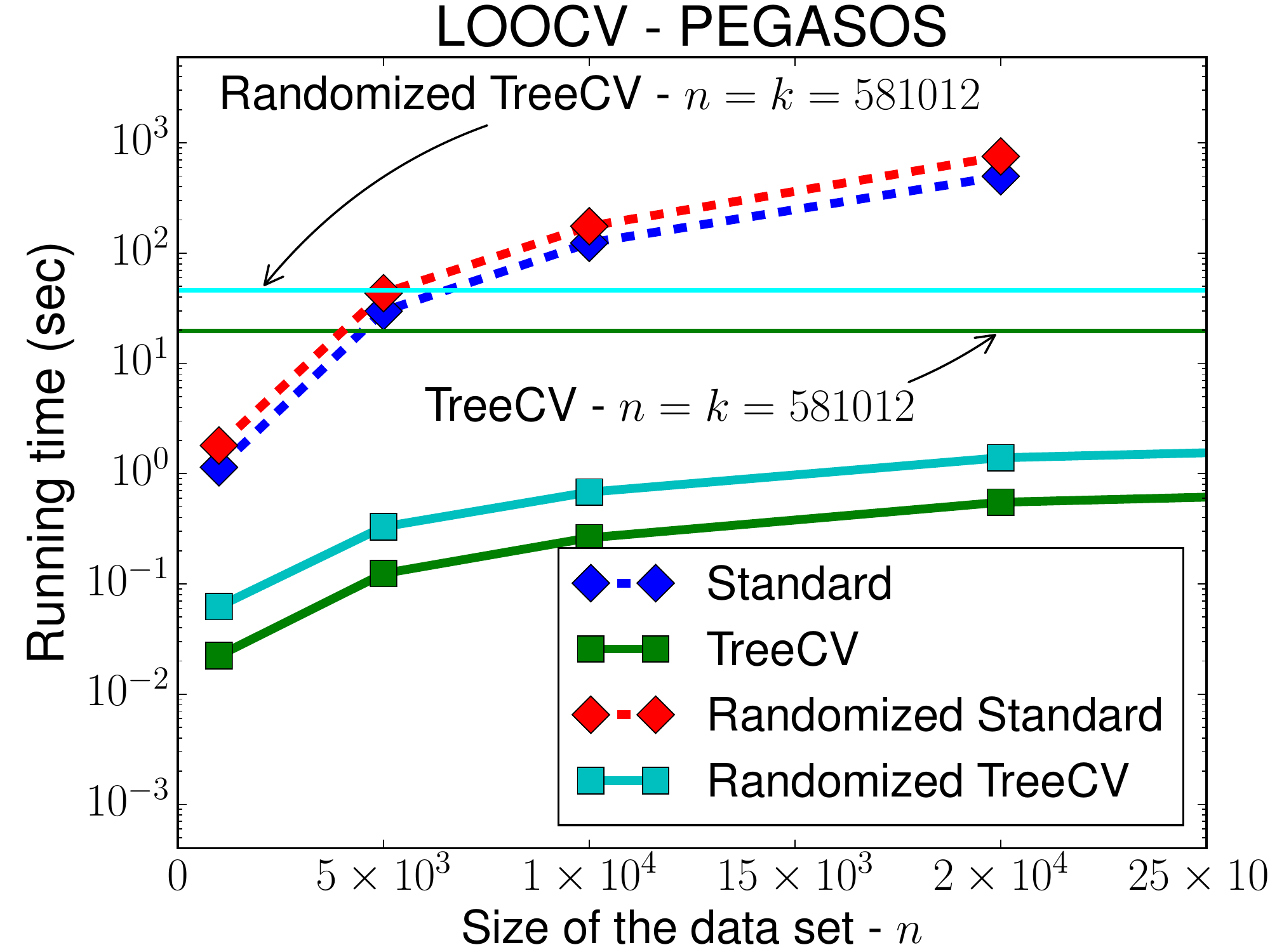}\bigskip\\
\includegraphics[width=\iw\textwidth]{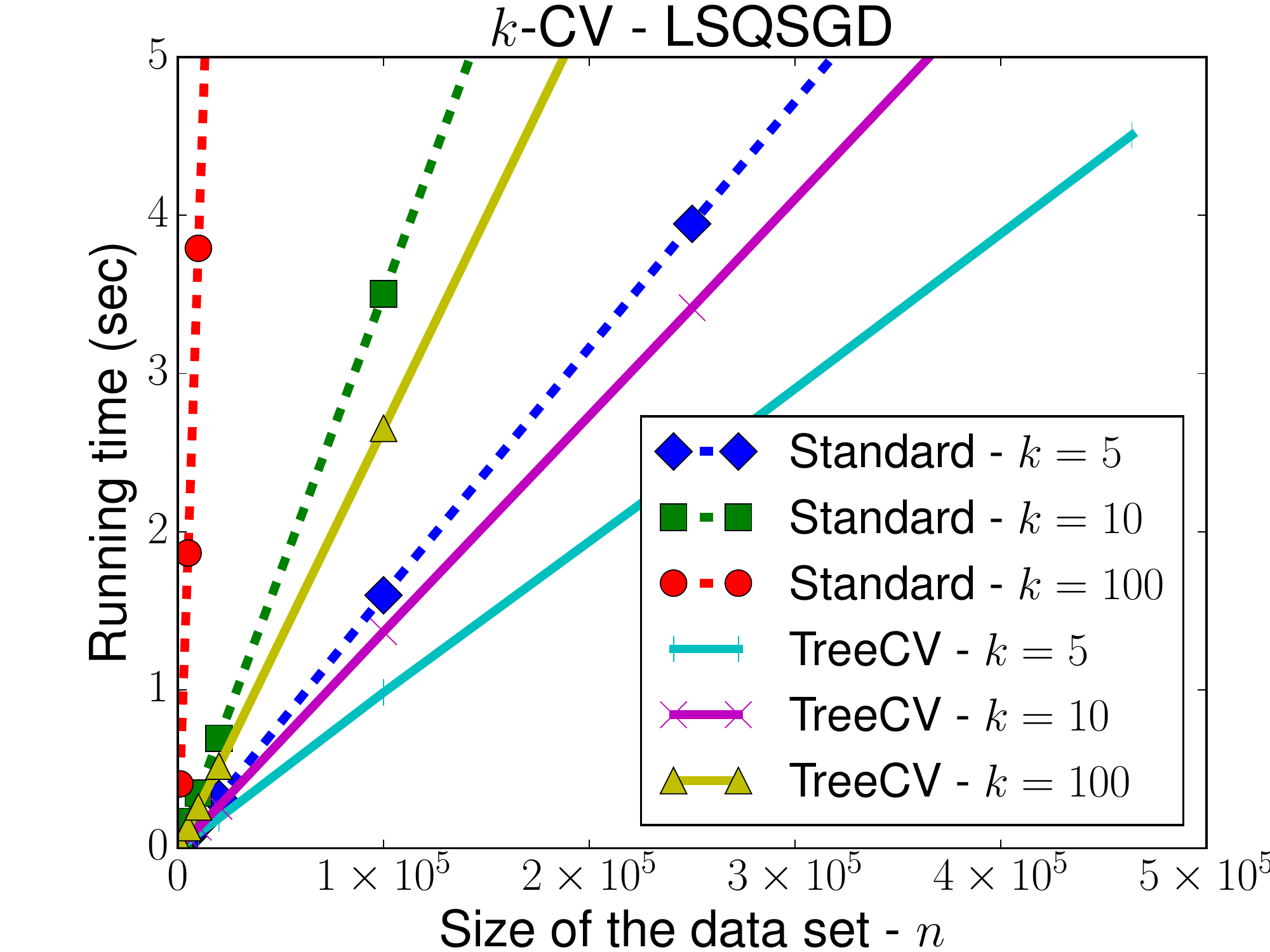}&
\includegraphics[width=\iw\textwidth]{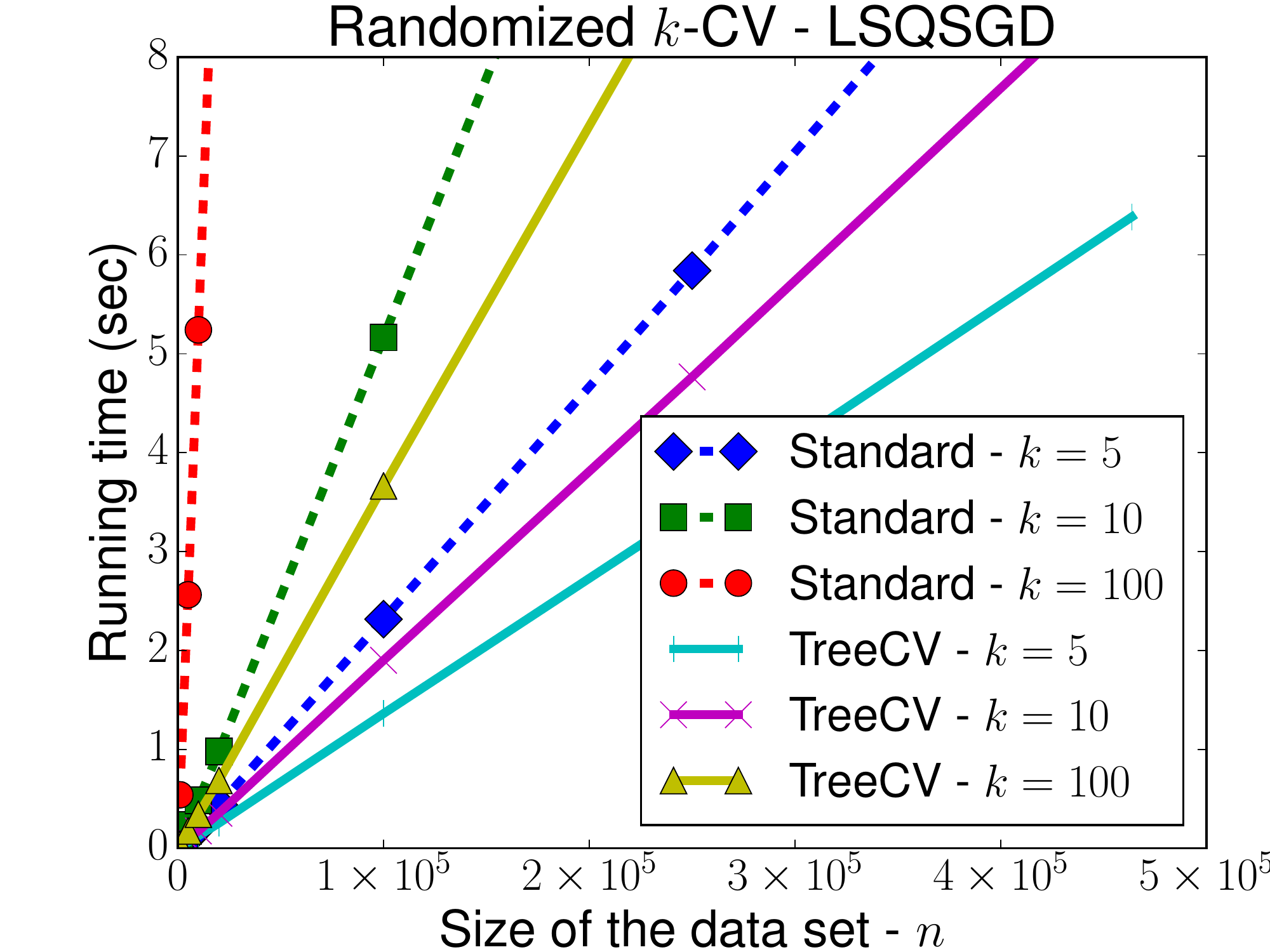}&
\includegraphics[width=\iw\textwidth]{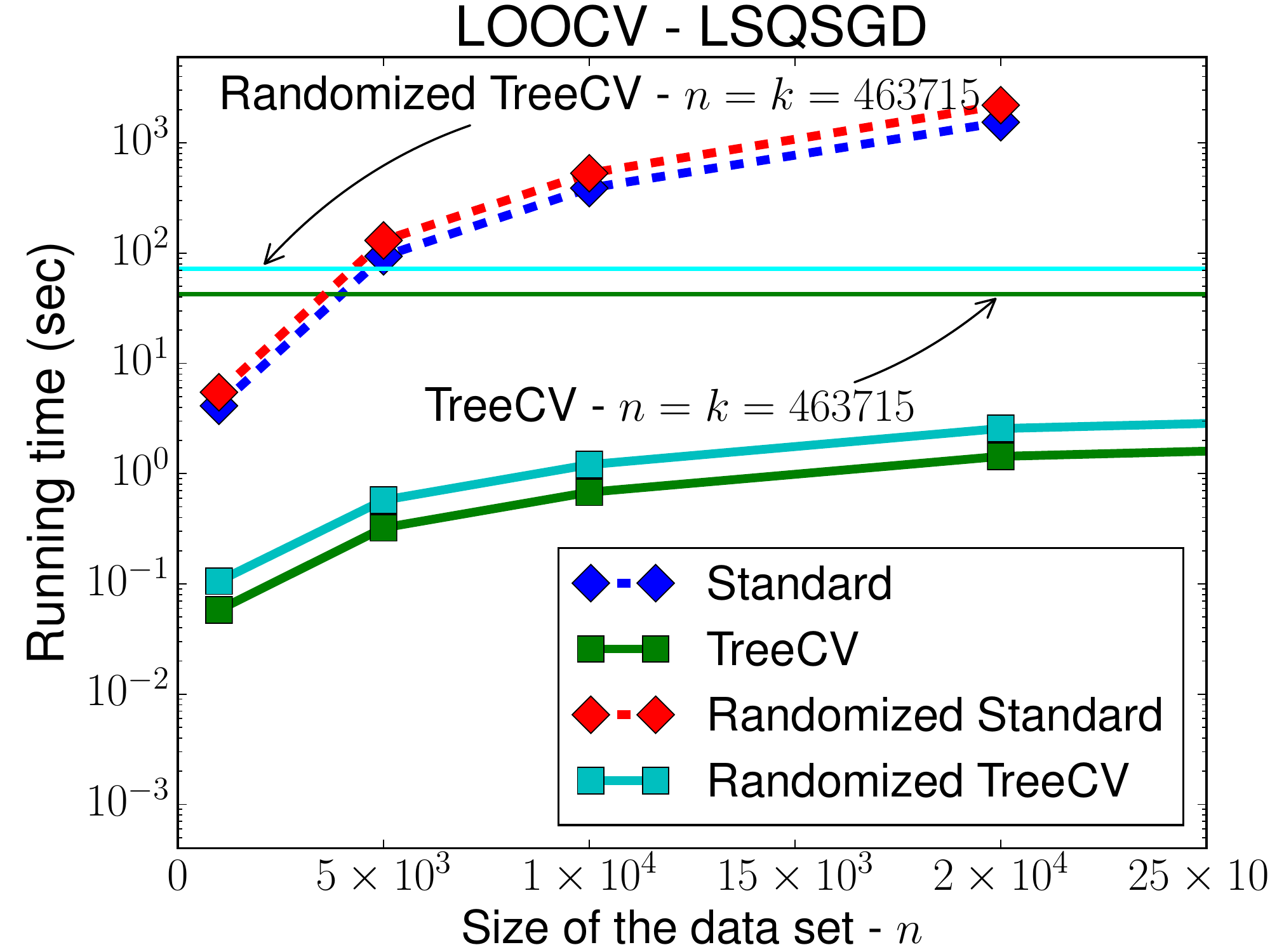} 
\end{tabular}
 \caption{Running time of $\cvAlg$ and standard $k$-CV for different  values of $k$ as a function of the number of data points $n$, averaged over $100$ independent repetitions. Top row: PEGASOS; bottom row: least-square SGD. Left column:
$k$-CV without permutations; middle column: $k$-CV with data
permutation; right column: LOOCV with and without permutations.
\label{fig:exp}
}
\end{figure*}

In this section we evaluate $\cvAlg$ and compare it with the standard ($k$-repetition) CV calculation.
We consider two incremental algorithms: linear
PEGASOS~\cite{shalev-shwartz2011pegasosprimalestimated} for SVM
classification, and least-square stochastic gradient descent (LSQSGD)
for linear least-squares regression (more precisely, LSQSGD is the
robust stochastic approximation algorithm of Nemirovski \etal\
\shortcite{nemirovski2009robust} for the squared loss and parameter vectors constrained in the unit $l_2$-ball). 
Following the suggestions in the original papers, we take the last hypothesis from PEGASOS and the average hypothesis from LSQSGD as our model.
We focus on the large-data
regime in which the algorithms learn from the data in a single pass.  

The algorithms were implemented in Python/Cython and Numpy.  The tests
were run on a single core of a computer with an Intel Xeon E5430
processor and 20 GB of RAM. We used datasets from the UCI
repository \cite{Lichman:2013}, downloaded from the LibSVM
website \cite{libSVM}.

We tested PEGASOS on the UCI Covertype dataset ($581{,}012$ data
points, $54$ features, $7$ classes), learning class ``1'' against the
rest of the classes.  The features were scaled to have unit
variance. The regularization parameter was set to $\lambda=10^{-6}$
following the suggestion of
Shalev-Shwartz \etal\ \shortcite{shalev-shwartz2011pegasosprimalestimated}. For LSQSGD, we used the
UCI YearPredictionMSD dataset ($463{,}715$ data points, $90$
features) and, following the suggestion of Nemirovski \etal\ \shortcite{nemirovski2009robust}, set the step-size to $\alpha=n^{-1/2}$. The target values where scaled to $[0,1]$. 

Naturally, PEGASOS and LSQSGD are sensitive to the order in which data points are
provided (although they are incrementally stable as mentioned after
Theorem~\ref{thm:excess-risk-stable}). 
In a vanilla implementation, the order of the data points is fixed in advance for the whole CV computation. That is, there is a fixed ordering of the chunks and of the samples within each chunk, and if we need to train a model with chunks $Z_{i_1},\ldots,Z_{i_j}$, the data points are given to the training algorithm according to this hierarchical ordering. 
This introduces certain dependence in the CV estimation procedure: for example, the model trained on chunks $Z_1,\ldots,Z_{k-1}$ has visited the data in a very similar order to the one trained on $Z_1,\ldots,Z_{k-2},Z_k$ (except for the last $n/k$ steps of the training). To eliminate this dependence, we also implemented a randomized version in which the samples used in a training phase are provided in a random order (that is, we take all the data points for the chunks $Z_{i_1},\ldots,Z_{i_j}$ to be used, and feed them to the training algorithm in a random order).

Table~\ref{tab:cvestimate} shows the values of the CV estimates computed under different scenarios. It can be observed that the standard ($k$-repetition) CV method is quite sensitive to the order of the points: the variance of the estimate does not really decay as the number of folds $k$ increases, while we see the expected decay for the randomized version. On the other hand, the non-randomized version of $\cvAlg$ does not show such a behavior, as the 
automatic re-permutation that happens during $\cvAlg$ might have made the $k$ folds less correlated. However, randomizing the order of the training points typically reduces the variance of the $\cvAlg$-estimate, as well.

\begin{table}[t]
\centering
\begin{adjustbox}{width=0.48\textwidth} 
\begin{tabular}{ccccc}
      \toprule
      \multicolumn{5}{c}{CV estimates for PEGASOS (misclassification rate $\times 100$)} \\
\hline  
& \multicolumn{2}{c}{$\cvAlg$} &  \multicolumn{2}{c}{Standard} \\
\hline
& fixed & randomized & fixed & randomized \\
\hline
$k=5$ &  $30.682 \pm 1.2127$ &  $30.839 \pm 0.9899$ &  $30.825 \pm 1.9248$ & $30.768 \pm 1.1243$ \\                                                             
$k=10$ &  $30.665 \pm 0.8299$ &  $30.554 \pm 0.7125$ &  $30.767 \pm 1.7754$ & $30.541 \pm 0.7993$ \\                                                            
$k=100$ &  $30.677 \pm 0.3040$ &  $30.634 \pm 0.2104$ &  $30.636 \pm 2.0019$ & $30.624 \pm 0.2337$ \\                                                           
$k = n$ &  $30.640 \pm 0.0564$ &  $30.637 \pm 0.0592$ &  N/A & N/A \\
\bottomrule
\end{tabular}
\end{adjustbox}
\bigskip \\
\begin{adjustbox}{width=0.48\textwidth} 
\begin{tabular}{ccccc}
      \toprule
      \multicolumn{5}{c}{CV estimates for LSQSGD (squared error $\times 100$)} \\
\hline  
 & \multicolumn{2}{c}{$\cvAlg$} &  \multicolumn{2}{c}{Standard} \\
\hline
& fixed & randomized & fixed & randomized \\
\hline
$k=5$ &  $25.299 \pm 0.0019$ &  $25.298 \pm 0.0018$ &  $25.299 \pm 0.0019$ &  $25.299 \pm 0.0017$ \\
$k=10$ &  $25.297 \pm 0.0016$ &  $25.297 \pm 0.0015$ &  $25.297 \pm 0.0016$ &  $25.297 \pm 0.0016$ \\
$k=100$ &  $25.296 \pm 0.0012$ &  $25.296 \pm 0.0013$ &  $25.296 \pm 0.0011$ &  $25.296 \pm 0.0013$ \\
$k = n$ &  $25.296 \pm 0.0012$ &  $25.296 \pm 0.0012$ &  N/A & N/A \\
\bottomrule
\end{tabular}
\end{adjustbox}
\caption{$k$-CV performance estimates averaged over $100$ repetitions (and their standard deviations), for the full datasets with and without data repermutation:  PEGASOS %
(top) and LSQSGD %
(bottom).
\label{tab:cvestimate}
}
\end{table}

Figure~\ref{fig:exp} shows the running times
of $\cvAlg$ and the standard CV method, as a function of $n$, for
PEGASOS (top row) and LSQSGD (bottom row). The first two columns show the running times
for different values of $k$, with and without randomizing the order of the data points (middle and left column, resp.), while the rightmost column shows the the running time
(log-scale) for LOOCV calculations. $\cvAlg$ outperforms the standard
method in all of the cases. It is notable that $\cvAlg$ makes the calculation of
LOOCV practical even for $n=581{,}012$, in a fraction of the time required by the standard method at $n=10{,}000$:
for example, for PEGASOS, TreeCV takes around $20$ seconds ($46$ when randomized) for computing LOOCV at $n=581{,}012$, while the standard method takes around $124$ seconds ($175$ when randomized) at $n=10{,}000$.
Furthermore, one can see that the variance reduction achieved by randomizing the data points comes at the price of a constant factor bigger running time (the factor is around $1.5$ for the standard method, and $2$  for $\cvAlg$). This comes from the fact that both the training time and the time of generating a random perturbation is linear in the number of points (assuming generating a random number uniformly from $\{1,\ldots,n\}$ can be done in constant time).

\section{Conclusion} %
\label{sec:concl-future-work}
We presented a general method, \cvAlg, to speed up cross-validation
for incremental learning algorithms. The method is applicable to a
wide range of supervised and unsupervised learning settings. We showed
that, under mild conditions on the incremental learning algorithm being
used, \cvAlg computes an accurate approximation of the $k$-CV estimate, and its running time scales logarithmically
in $k$ (the number of CV folds), while the running time of the standard method of training
$k$ separate models scales linearly with $k$.

Experiments on classification and regression, using two well-known incremental learning algorithms, PEGASOS and least-square SGD, confirmed the
speedup and predicted accuracy. When the model learned by the learning algorithm
depends on whether the data is provided incrementally or in batch (or on the order of the data, as in the case of online algorithms), the
CV estimate calculated by our method was still close to the CV
computed by the standard method, but with a lower variance.

\section*{Acknowledgments}

This work was supported by the Alberta Innovates Technology Futures
and NSERC.

{\small\bibliography{tree-cv} \bibliographystyle{named}}

\end{document}